%% file: main.tex
\documentclass[sigconf]{acmart}
\usepackage{amsmath}
\usepackage{hhline}
\usepackage{multirow}
\usepackage{bbm}
\usepackage{makecell}
\usepackage{amsfonts}
\usepackage{booktabs}
\usepackage{mathtools,xcolor}
\usepackage{algorithm}
\usepackage{algorithmic}
\usepackage{float}
\AtBeginDocument{%
  \providecommand\BibTeX{{%
    \normalfont B\kern-0.5em{\scshape i\kern-0.25em b}\kern-0.8em\TeX}}}

\newcommand{\ariel}{{{\textsc{ArieL}}}}

\copyrightyear{2022} 
\acmYear{2022} 
\setcopyright{acmcopyright}\acmConference[WWW '22]{Proceedings of the ACM Web Conference 2022}{April 25--29, 2022}{Virtual Event, Lyon, France}
\acmBooktitle{Proceedings of the ACM Web Conference 2022 (WWW '22), April 25--29, 2022, Virtual Event, Lyon, France}
\acmPrice{15.00}
\acmDOI{10.1145/3485447.3512180}
\acmISBN{978-1-4503-9096-5/22/04}

\begin{document}
\copyrightyear{2022}
\acmYear{2022}
\setcopyright{acmcopyright}
\acmConference[WWW '22] {Proceedings of the ACM Web Conference 2022}{April 25--29, 2022}{Virtual Event, Lyon, France.}
\acmBooktitle{Proceedings of the ACM Web Conference 2022 (WWW '22), April 25--29, 2022, Virtual Event, Lyon, France}
\acmPrice{15.00}
\acmISBN{978-1-4503-9096-5/22/04}
\acmDOI{10.1145/3485447.3512183}

\title{Adversarial Graph Contrastive Learning \\ with Information Regularization}

\author{Shengyu Feng}
\email{shengyu8@illinois.edu}
\orcid{1234-5678-9012}
\affiliation{%
  \institution{University of Illinois at Urbana-Champaign}
  \country{USA}
}
\author{Baoyu Jing}
\email{baoyuj2@illinois.edu}
\orcid{1234-5678-9012}
\affiliation{%
  \institution{University of Illinois at Urbana-Champaign}
  \country{USA}
}
\author{Yada Zhu}
\email{yzhu@us.ibm.com}
\orcid{1234-5678-9012}
\affiliation{%
  \institution{IBM Research}
  \country{USA}
}

\author{Hanghang Tong}
\email{htong@illinois.edu}
\orcid{1234-5678-9012}
\affiliation{%
  \institution{University of Illinois at Urbana-Champaign}
  \country{USA}
}

\begin{abstract}
\input{abstract.tex}
\end{abstract}





\maketitle

\section{Introduction}
\input{introduction}

\section{Problem Definition}
\input{problem}

\section{Method}
\input{method}

\section{Experiments}
\input{experiment}

\section{Related Work}
\input{related}

\section{Conclusion}
\input{conclusion}
\begin{acks}
This work is supported by National Science Foundation under grant No. 1947135,
    DARPA HR001121C0165,
Department of Homeland Security under Grant Award Number 17STQAC00001-03-03, NIFA award 2020-67021-32799 and Army Research Office (W911NF2110088).
The content of the information in this document does not necessarily reflect the position or the policy of the Government, and no official endorsement should be inferred.  The U.S. Government is authorized to reproduce and distribute reprints for Government purposes notwithstanding any copyright notation here on.
\end{acks}
\newpage
\bibliographystyle{ACM-Reference-Format}
\bibliography{main}


\appendix
\clearpage
\input{appendix}
\end{document}

%% file: abstract.tex
Contrastive learning is an effective unsupervised method in graph representation learning. Recently, the data augmentation based contrastive learning method has been extended from images to graphs. However, most prior works are directly adapted from the models designed for images. Unlike the data augmentation on images, the data augmentation on graphs is far less intuitive and much harder to provide high-quality contrastive samples, which are the key to the performance of contrastive learning models. This leaves much space for improvement over the existing graph contrastive learning frameworks. In this work, by introducing an adversarial graph view and an information regularizer, we propose a simple but effective method, \textit{Adversarial Graph Contrastive Learning} (\ariel), to extract informative contrastive
samples within a reasonable constraint. It consistently outperforms the current graph contrastive learning methods in the node classification task over various real-world datasets and further improves the robustness of graph contrastive learning. The code is at \url{https://github.com/Shengyu-Feng/ARIEL}.

%% file: introduction.tex
Contrastive learning is a widely used technique in various graph representation learning tasks. In contrastive learning, the model tries to minimize the distances among positive pairs and maximize the distances among negative pairs in the embedding space. The definition of positive and negative pairs is the key component in contrastive learning. Earlier methods like DeepWalk \cite{Perozzi:2014:DOL:2623330.2623732} and node2vec \cite{grover2016node2vec} define positive and negative pairs based on the co-occurrence of node pairs in the random walks. For knowledge graph embedding, it is a common practice to define positive and negative pairs based on translations \cite{NIPS2013_1cecc7a7, 10.5555/2893873.2894046, 10.5555/2886521.2886624, ji-etal-2015-knowledge, yan2021dynamic, wang2018acekg}.   

Recently, the breakthroughs of contrastive learning in computer vision have inspired some works to apply similar ideas from visual representation learning to graph representation learning. 
To name a few, Deep Graph Infomax (DGI) \cite{velickovic2018deep} extends Deep InfoMax \cite{hjelm2019learning} and achieves significant improvements over previous random-walk-based methods. Graphical Mutual Information (GMI) \cite{peng2020graph} uses the same framework as DGI but generalizes the concept of mutual information from vector space to graph domain. Contrastive multi-view graph representation learning (referred to as MVGRL in this paper) \cite{hassani2020contrastive} further improves DGI by introducing graph diffusion into the contrastive learning framework.  The more recent works often follow the data augmentation-based contrastive learning methods \cite{he2020momentum, chen2020simple}, which treat the data-augmented samples from the same instance as positive pairs and different instances as negative pairs. Graph Contrastive Coding (GCC) \cite{Qiu_2020} uses random walks with restart \cite{tong2006fast} to generate two subgraphs for each node as two data-augmented samples.
Graph Contrastive learning with Adaptive augmentation (GCA) \cite{Zhu_2021} introduces an adaptive data augmentation method that perturbs both the node features and edges according to their importance, and it is trained in a similar way as the famous visual contrastive learning framework SimCLR \cite{chen2020simple}. Its preliminary work, which uses uniform random sampling rather than adaptive sampling, is referred to as GRACE \cite{zhu2020deep} in this paper.
Robinson et al. \cite{robinson2021contrastive} propose a way to select hard negative samples based on the embedding space distances, and use it to obtain high-quality graph embedding. There are also many works \cite{you2020graph, zhao2020data} systemically studying the data augmentation on the graphs.
 
However, unlike the transformations on images, the transformations on graphs are far less intuitive to human beings. The data augmentation on the graph could be either too similar to or totally different from the original graph. This, in turn, leads to a crucial question, that is, {\em how to generate a new graph that is hard enough for the model to discriminate from the original one, and in the meanwhile also maintains the desired properties?} 

Inspired by some recent works
\cite{kim2020adversarial, jiang2020robust, ho2020contrastive, jovanovi2021robust, suresh2021adversarial}, 
we introduce the adversarial training on the graph contrastive learning and propose a new framework called \textit{\underline{A}dversarial G\underline{R}aph Contrast\underline{I}v\underline{E} \underline{L}earning} (\ariel). Through the adversarial attack on both topology and node features, we generate an adversarial sample from the original graph. On the one hand, since the perturbation is under the constraint, the adversarial sample still stays close enough to the original one. On the other hand, the adversarial attack makes sure the adversarial sample is hard to discriminate from the other view by increasing the contrastive loss. On top of that, we propose a new constraint
called information regularization which could stabilize the training of \ariel\ and prevent the collapsing. We demonstrate that the proposed \ariel\  outperforms the existing graph contrastive learning frameworks in the node classification task on both real-world graphs and adversarially attacked graphs. 

In summary, we make the following contributions:
\begin{itemize}
    \item We introduce the adversarial view as a new form of data augmentation in graph contrastive learning.
    \item We propose information regularization to stabilize adversarial graph contrastive learning.
    \item We empirically demonstrate that \ariel\ can achieve better performance and higher robustness compared with previous graph contrastive learning methods.
\end{itemize}

The rest of the paper is organized as follows. Section 2 gives the problem definition of graph representation learning and the preliminaries. Section 3 describes the proposed algorithm. The experimental results are presented in Section 4. After reviewing related work in Section 5, we conclude the paper in Section 6.

%% file: problem.tex
In this section, we will introduce all the notations used in this paper and give a formal definition of our problem. Besides, we briefly introduce the preliminaries of our method.

\subsection{Graph Representation Learning}
For graph representation learning, let each node $v_i$ have a $d$-dimensional feature $\mathbf{X}[i,:]$, and all edges are assumed to be unweighted and undirected. We use a binary adjacency matrix $\mathbf{A}\in \{0,1\}^{n\times n}$ to represent the nodes and edges information, where $\mathbf{A}[i,j]=1$ if and only if the node pair $(v_i,v_j) \in \mathcal{E}$. In the following text, we will use $G=\{\mathbf{A},\mathbf{X}\}$ to represent the graph. 

The objective of the graph representation learning is to learn an encoder $f: \mathbb{R}^{n\times n}\times\mathbb{R}^{n\times d}\rightarrow\mathbb{R}^{n\times d'}$, which maps the nodes in the graph into low-dimensional embeddings. Denote the node embedding matrix $\mathbf{H}=f(\mathbf{A},\mathbf{X})$, where $\mathbf{H}[i,:]\in \mathbb{R}^{d'}$ is the embedding for node $v_i$. This representation could be used for downstream tasks like node classification.
\subsection{InfoNCE Loss}
InfoNCE loss \cite{oord2019representation} is the predominant work-horse of the contrastive learning loss, which maximizes the lower bound of the mutual information between two random variables. For each positive pair $(\mathbf{x},\mathbf{x}^+)$ associated with $k$ negative samples of $\mathbf{x}$, denoted as $\{\mathbf{x}_1^-, \mathbf{x}_2^-, \cdots,\mathbf{x}_k^-\}$, InfoNCE loss could be written as
\begin{align}
    L_k = -\log{(\frac{g(\mathbf{x},\mathbf{x}^+)}{g(\mathbf{x},\mathbf{x}^+)+\sum\limits_{i=1}^kg(\mathbf{x},\mathbf{x}_i^-)})}.
\end{align}
Here $g(\cdot)$ is the density ratio with the property that $ g(\mathbf{a},\mathbf{b}) \propto \frac{p(\mathbf{a}|\mathbf{b})}{p(\mathbf{a})}$, 
where $\propto$ stands for \textit{proportional to}. It has been shown by \cite{oord2019representation} that $-L_k$ actually serves as the lower bound of the mutual information $I(\mathbf{x};\mathbf{x}^+)$ with
\begin{align}
     I(\mathbf{x};\mathbf{x}^+)\geq \log{(k)}-L_k.
\end{align}
\subsection{Graph Contrastive Learning}
We build the proposed method upon the framework of SimCLR \cite{chen2020simple} since it is the state of the art contrastive learning method,
which is also the basic framework that GCA \cite{Zhu_2021} is built on. Given a graph $G$, two views of the graph $G_1=\{\mathbf{A}_1,\mathbf{X}_1\}$ and $G_2=\{\mathbf{A}_2,\mathbf{X}_2\}$ are first generated. This step can be treated as the data augmentation on the original graph, and various augmentation methods can be used herein. We use random edge dropping and feature masking as GCA does. The node embedding matrix for each graph can be computed as $\mathbf{H}_1=f(\mathbf{A}_1,\mathbf{X}_1)$ and $\mathbf{H}_2=f(\mathbf{A}_2,\mathbf{X}_2)$. The corresponding node pairs in two graph views are the positive pairs and all other node pairs are negative. Define $\theta(\mathbf{u},\mathbf{v})$ to be the similarity function between vectors $\mathbf{u}$ and $\mathbf{v}$, in practice, it's usually chosen as the cosine similarity on the projected embedding of each vector, using a two-layer neural network as the projection head. Denote $\mathbf{u}_i=\mathbf{H}_1[i,:]$ and $\mathbf{v}_i=\mathbf{H}_2[i,:]$, the contrastive loss is defined as
\begin{align}
\label{eq:contra}
L_{\text{con}}(G_1,G_2) = \frac{1}{2n}\sum_{i=1}^n(l(\mathbf{u}_i,\mathbf{v}_i)+l(\mathbf{v}_i,\mathbf{u}_i)),
\end{align}
\begin{align}
\nonumber
&l(\mathbf{u}_i,\mathbf{v}_i) =\\
&-\log{\frac{e^{\theta(\mathbf{u}_i,\mathbf{v}_i)/\tau}}{e^{\theta(\mathbf{u}_i,\mathbf{v}_i)/\tau} +\sum\limits_{j\neq i}e^{\theta(\mathbf{u}_i,\mathbf{v}_j)/\tau}+\sum\limits_{j\neq i}e^{\theta(\mathbf{u}_i,\mathbf{u}_j)/\tau}}},
\end{align}
where $\tau$ is a temperature parameter. $l(\mathbf{v}_i,\mathbf{u}_i)$ is symmetrically defined by exchanging the variables in $l(\mathbf{u}_i,\mathbf{v}_i)$. This loss is basically a variant of InfoNCE loss which is symmetrically defined instead.

In principle, our framework could be applied on any graph neural network (GNN) architecture as long as it could be attacked. For simplicity,
we employ a two-layer Graph Convolutional Network (GCN) \cite{kipf2017semisupervised} in this work. Define the symmetrically normalized adjacency matrix
\begin{align}
\hat{\mathbf{A}} =  \Tilde{\mathbf{D}}^{-\frac{1}{2}}\Tilde{\mathbf{A}}\Tilde{\mathbf{D}}^{-\frac{1}{2}}, 
\end{align}
where $\Tilde{\mathbf{A}} = \mathbf{A} +\mathbf{I}_n$ is the adjacency matrix with self-connections added and $\mathbf{I}_n$ is the identity matrix, $\Tilde{\mathbf{D}}$ is the diagonal degree matrix of $\Tilde{\mathbf{A}}$ with $\Tilde{\mathbf{D}}[i,i] = \sum_j\Tilde{\mathbf{A}}[i,j]$.   
The two-layer GCN is given as 
\begin{align}
    f(\mathbf{A},\mathbf{X}) = \sigma(\hat{\mathbf{A}}\sigma(\hat{\mathbf{A}}\mathbf{X}\mathbf{W}^{(1)})\mathbf{W}^{(2)}),
\end{align}
where $\mathbf{W}^{(1)}$ and $\mathbf{W}^{(2)}$ are the weights of the first and second layer respectively, $\sigma(\cdot)$ is the activation function.   

\subsection{Projected Gradient Descent Attack}
Projected Gradient Descent (PGD) attack \cite{madry2019deep} is an iterative attack method that projects the perturbation onto the ball of interest at the end of each iteration. Assuming that the loss $L(\cdot)$ is a function of the input matrix $\mathbf{Z} \in \mathbb{R}^{n\times d}$, at $t$-th iteration, the perturbation matrix $\mathbf{\Delta}_t\in\mathbb{R}^{n\times d}$ under an $l_{\infty}$-norm constraint could be written as
\begin{align}
    \mathbf{\Delta}_{t} = \Pi_{\|\mathbf{\Delta}\|_{\infty}\leq \delta} (\mathbf{\Delta}_{t-1}+ \eta\cdot \text{sgn}(\nabla_{\mathbf{\Delta}}L(\mathbf{Z}+\mathbf{\Delta}_{t-1})),
\end{align}
where $\eta$ is the step size, $\text{sgn}(\cdot)$ takes the sign of each element in the input matrix, and $\Pi_{\|\mathbf{\Delta}\|_{\infty}\leq \delta}$ projects the perturbation onto the $\delta$-ball in the $l_{\infty}$-norm. 

%% file: method.tex
In this section, we will first investigate the vulnerability of the graph contrastive learning, then we will spend the remaining section discussing each part of \ariel\ in detail.
\subsection{Vulnerability of the Graph Contrastive Learning}

Many GNNs are known to be vulnerable to adversarial attacks \cite{bojchevski2019adversarial, Z_gner_2018}, so we first investigate the vulnerability of the graph neural networks trained with the contrastive learning objective in Equation (\ref{eq:contra}). We generate a sequence of $60$ 
graphs by iteratively dropping edges and masking the features. Let $G_0=G$, for the $t$-th iteration, we generate $G_t$ from $G_{t-1}$ by randomly dropping the edges in $G_{t-1}$ and randomly masking the unmasked features, both with probability $p=0.03$. 
Since $G_t$ is guaranteed to contain less information than $G_{t-1}$, $G_t$ should be less similar to $G_0$ than $G_{t-1}$, on both the graph and node level. Denote the node embeddings of $G_t$ as $\mathbf{H}_t$, we measure the similarity $\theta(\mathbf{H}_t[i,:], \mathbf{H}_0[i,:])$, and it is expected that the similarity decreases as the iteration goes on. 

We generate the sequences on two datasets, \textit{Amazon-Computers} and \textit{Amazon-Photo} \cite{shchur2019pitfalls}, and the results are shown in Figure \ref{fig:simi1}.
\begin{figure}
    \centering
    \includegraphics[width=.8\linewidth]{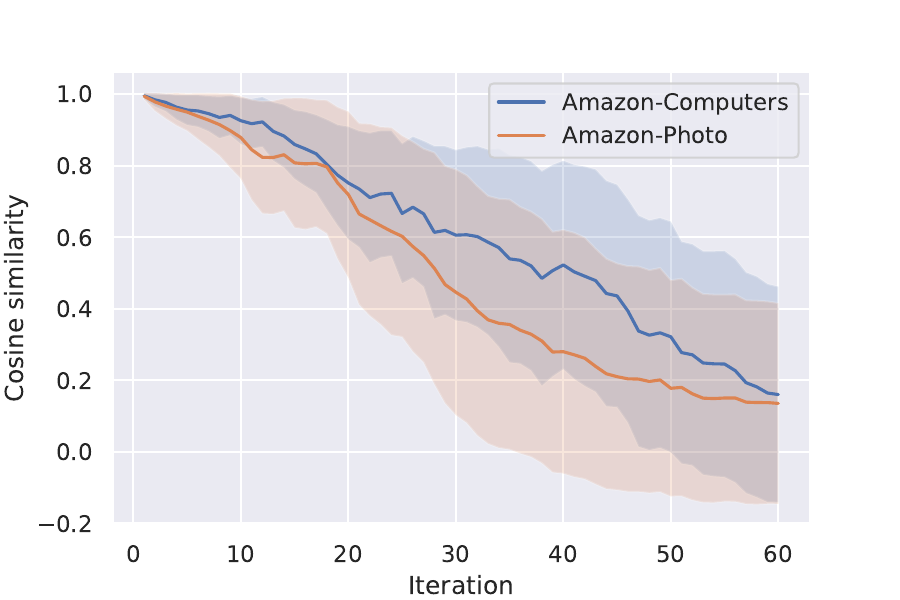}
    \caption{Average cosine similarity between the node embeddings of the original graph and the perturbed graph, results are on datasets Amazon-Computers and Amazon-Photo. The shaded area represents the standard deviation.}
    \label{fig:simi1}
\end{figure}
At the $30$-th iteration, with $0.97^{30}=40.10\%$ edges and features left, the average similarity of the positive samples is under $0.5$ on Amazon-Photo. At the $60$-th iteration, with $0.97^{60}=16.08\%$ edges and features left, the average similarity drops under $0.2$ on both Amazon-Computers and Amazon-Photo. Additionally, starting from the $30$-th iteration, the cosine similarity has around $0.3$ standard deviation for both datasets, which indicates that a lot of nodes are actually very sensitive to the external perturbations, even if we do not add any adversarial component but just mask out some information. These results demonstrate that the current graph contrastive learning framework is not trained over enough high-quality contrastive samples and is not robust to adversarial attacks.

Given this observation, we are motivated to build an adversarial graph contrastive learning framework that could improve the performance and robustness of the previous graph contrastive learning methods. The overview of our framework is shown in Figure \ref{fig:architecture}.
\begin{figure*}
    \centering
    \includegraphics[width=0.95\textwidth]{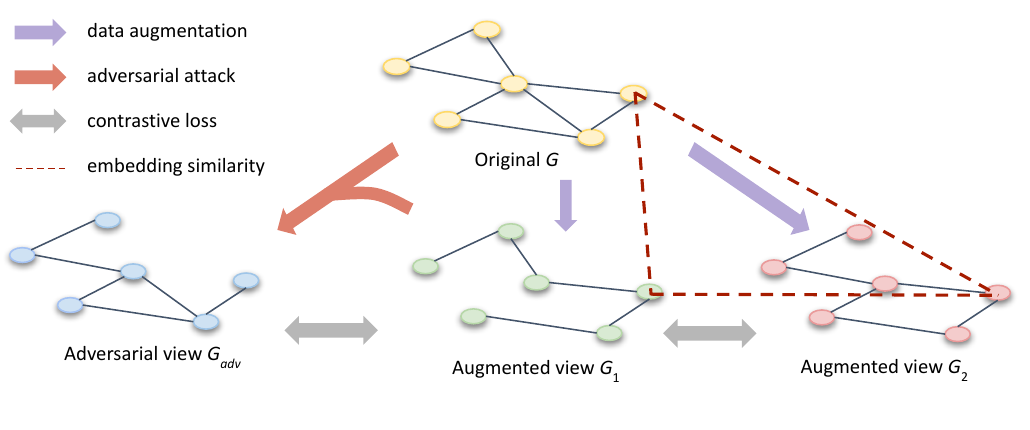}
    \caption{The overview of the proposed \ariel\ framework. For each iteration, two augmented views are generated from the original graph by data augmentation (purple arrows), and then an adversarial view is generated (red arrow) from the original graph by maximizing the contrastive loss against one of the augmented views. Besides, the similarities of the corresponding nodes (dashed lines) will get penalized by the information regularization if they exceed the estimated upper bound. The objective of \ariel\ is to minimize the contrastive loss (grey arrows) between the augmented views, the adversarial view, the corresponding augmented view, and the information regularization. Best viewed in color.}
    \label{fig:architecture}
\end{figure*}
\subsection{Adversarial Training}

Although most existing attack frameworks are targeted at supervised learning, it is natural to generalize these methods to contrastive learning by replacing the classification loss with the contrastive loss. The goal of the adversarial attack on contrastive learning is to maximize the contrastive loss by adding a small perturbation to the graph, which can be formulated as 
\begin{align}
    \label{eq:adv_sample}
    G_{\text{adv}} = \arg\max_{G'} L_{\text{con}}(G_1, G'), 
\end{align}
where $G'=\{\mathbf{A}', \mathbf{X}'\}$ is generated from the original graph $G$, and the change is constrained by the budget $\Delta_{\mathbf{A}}$ and $\Delta_{\mathbf{X}}$ as
\begin{align}
    &\sum_{i,j}|\mathbf{A}'[i,j]-\mathbf{A}[i,j]|\leq\Delta_{\mathbf{A}},\\
    &\sum_{i,j}|\mathbf{X}'[i,j]-\mathbf{X}[i,j]|\leq\Delta_{\mathbf{X}}.
\end{align}
We treat adversarial attacks as one kind of data augmentation. Although we find it effective to make the adversarial attack on one or two augmented views as well, we follow the typical contrastive learning procedure as in SimCLR \cite{chen2020simple} to make the attack on the original graph in this work. Besides, it does not matter whether $G_1$, $G_2$, or $G$ is chosen as the anchor for the adversary, each choice works in our framework and it can also be sampled as a third view. In our experiments, we use PGD attack \cite{madry2019deep} as our attack method. 

We generally follow the method proposed by Xu et al. \cite{xu2019topology} to make the PGD attack on the graph structure and apply the regular PGD attack method on the node features. Define the supplement of the adjacency matrix as $\mathbf{\bar{A}}=\mathbf{1}_{n\times n}-\mathbf{I}_n-\mathbf{A}$, where $\mathbf{1}_{n\times n}$ is the ones matrix of size $n\times n$. The perturbed adjacency matrix can be written as
\begin{align}
    \mathbf{A}_{\text{adv}} &= \mathbf{A}+ \mathbf{C}\circ \mathbf{L}_{\mathbf{A}},\\
    \mathbf{C} &= \bar{\mathbf{A}} - \mathbf{A},
\end{align}
where $\circ$ is the element-wise product, and $\mathbf{L}_{\mathbf{A}}\in\{0,1\}^{n\times n}$ with each element $\mathbf{L}_{\mathbf{A}}[i,j]$ corresponding to the modification (e.g., add, delete or no modification) of the edge between the node pair $(v_i,v_j)$. The perturbation on $\mathbf{X}$ follows the regular PGD attack procedure and the perturbed feature matrix can be written as 
\begin{align}
    \mathbf{X}_{\text{adv}} = \mathbf{X} + \mathbf{L}_{\mathbf{X}},
\end{align}
where $\mathbf{L}_{\mathbf{X}}\in\mathbb{R}^{n\times d}$ is the perturbation on the feature matrix.

For the ease of optimization, $\mathbf{L}_{\mathbf{A}}$ is relaxed to its convex hull $\Tilde{\mathbf{L}}_{\mathbf{A}}\in[0,1]^{n\times n}$, which satisfies $\mathcal{S}_{\mathbf{A}}=\{\Tilde{\mathbf{L}}_{\mathbf{A}}|\sum\limits_{i,j}\Tilde{\mathbf{L}}_{\mathbf{A}}\leq \Delta_{\mathbf{A}}, \Tilde{\mathbf{L}}_{\mathbf{A}}\in[0,1]^{n\times n}\}$. The constraint on $\mathbf{L}_{\mathbf{X}}$ can be written as $\mathcal{S}_{\mathbf{X}}=\{\mathbf{L}_{\mathbf{X}}|\| \mathbf{L}_{\mathbf{X}}\|_{\infty}\leq\delta_{\mathbf{X}}, \mathbf{L}_{\mathbf{X}}\in\mathbb{R}^{n\times d}\}$, where we directly treat $\delta_{\mathbf{X}}$ as the constraint on the feature perturbation. In each iteration, we make the updates
\begin{align}
    \label{eq:pgd_a}
    \Tilde{\mathbf{L}}_{\mathbf{A}}^{(t)} &= \Pi_{\mathcal{S}_{\mathbf{A}}}[\Tilde{\mathbf{L}}_{\mathbf{A}}^{(t-1)}+\alpha\cdot\mathbf{G}_{\mathbf{A}}^{(t)}],\\
    \label{eq:pgd_x}
    \mathbf{L}_{\mathbf{X}}^{(t)} & =\Pi_{\mathcal{S}_{\mathbf{X}}}[\mathbf{L}_{\mathbf{X}}^{(t-1)}+\beta\cdot\text{sgn}(\mathbf{G}_{\mathbf{X}}^{(t)})],
\end{align}
where $t$ denotes the current number of iterations, and 
\begin{align}
\mathbf{G}_{\mathbf{A}}^{(t)}&=\nabla_{\Tilde{\mathbf{L}}_{\mathbf{A}}}   L_{\text{con}}(G_1,G_{\text{adv}}^{(t-1)}),\\ \mathbf{G}_{\mathbf{X}}^{(t)}&=\nabla_{\mathbf{L}_{\mathbf{X}}}  L_{\text{con}}(G_1,G_{\text{adv}}^{(t-1)}),
\end{align}
denote the gradients of the loss with respect to $\Tilde{\mathbf{L}}_{\mathbf{A}}$ and $\mathbf{L}_{\mathbf{X}}$  respectively. Here $G_{\text{adv}}^{(t-1)}$ is defined as $\{\mathbf{A}+ \mathbf{C}\circ \Tilde{\mathbf{L}}_{\mathbf{A}}^{(t-1)},  \mathbf{X} + \mathbf{L}_{\mathbf{X}}^{(t-1)}\}$.
The projection operation $\Pi_{\mathcal{S}_{\mathbf{X}}}$  simply clips $\mathbf{L}_{\mathbf{X}}$ into the range $[-\delta_{\mathbf{X}}, \delta_{\mathbf{X}}]$. The projection operation  $\Pi_{\mathcal{S}_{\mathbf{A}}}$ is calculated as  
\begin{align}
    \Pi_{\mathcal{S}_{\mathbf{A}}}(\mathbf{Z}) = \begin{cases}
    P_{[0,1]}[\mathbf{Z}-\mu \mathbf{1}_{n\times n}], & \text{if } \mu>0, \text{and}\\& \sum\limits_{i,j}{P_{[0,1]}[\mathbf{Z}-\mu \mathbf{1}_{n\times n}]}=\Delta_{\mathbf{A}},\\
    \\
     P_{[0,1]}[\mathbf{Z}], & \text{if } \sum\limits_{i,j}{P_{[0,1]}[\mathbf{Z}]}\leq\Delta_{\mathbf{A}},
    \end{cases}
\end{align}
where $P_{[0,1]}[\mathbf{Z}]$ clips $\mathbf{Z}$ into the range $[0,1]$. We use the bisection method \cite{boyd_vandenberghe_2004} to solve the equation $\sum_{i,j} P_{[0,1]}[\mathbf{Z}-\mu \mathbf{1}_{n\times n}]=\Delta_{\mathbf{A}}$ with respect to the dual variable $\mu$.

To finally obtain $\mathbf{L}_{\mathbf{A}}$ from $\Tilde{\mathbf{L}}_{\mathbf{A}}$, each element is independently sampled from a Bernoulli distribution as $ \mathbf{L}_{\mathbf{A}}[i,j]\sim \text{Bernoulli}( \Tilde{\mathbf{L}}_{\mathbf{A}}[i,j])$.

\subsection{Adversarial Graph Contrastive Learning}
To assimilate the graph contrastive learning and adversarial training together, we treat the adversarial view $G_{\text{adv}}$ obtained from Equation (\ref{eq:adv_sample}) as another view of the graph. We define the adversarial contrastive loss as the contrastive loss between $G_1$ and $G_{\text{adv}}$. The adversarial contrastive loss is added to the original contrastive loss in Equation (\ref{eq:contra}), which becomes
\begin{align}\label{eq:obj}
L(G_1,G_2, G_{\text{adv}}) = L_{\text{con}}(G_1,G_2)+\epsilon_1 L_{\text{con}}(G_1,G_{\text{adv}}),
\end{align}
where $\epsilon_1 >0$ is the adversarial contrastive loss coefficient. We further adopt two additional subtleties on top of this basic framework: subgraph sampling and curriculum learning.  For each iteration, a subgraph $G_s$ with a fixed size is first sampled from the original graph $G$, then the data augmentation and adversarial attack are both conducted on this subgraph. The subgraph sampling could avoid the gradient derivation on the whole graph, which will lead to heavy computation on a large network. And for every $T$ epochs, the adversarial contrastive loss coefficient is multiplied by a weight $\gamma$. When $\gamma>1$, the portion of the adversarial contrastive loss is gradually increasing and the contrastive learning becomes harder as the training goes on.

\subsection{Information Regularization}
Adversarial training could effectively improve the model's robustness to the perturbations, nonetheless, we find these hard training samples could impose the additional risk of training collapsing, i.e., the model will be located at a bad parameter area at the early stage of the training, assigning higher probability to a highly perturbed sample than a mildly perturbed one. In our experiment, we find this vanilla adversarial training method may fail to converge in some cases (e.g., Amazon-Photo dataset). To stabilize the training, we add one constraint termed {\em information regularization}.

The data processing  inequality \cite{Cover2006} states that for three random variables $\mathbf{Z}_1$, $\mathbf{Z}_2$ and $\mathbf{Z}_3\in\mathbb{R}^{n\times d'}$, if they satisfy the Markov relation $\mathbf{Z_1}\rightarrow\mathbf{Z}_2\rightarrow\mathbf{Z}_3$, then the inequality $I(\mathbf{Z}_1;\mathbf{Z}_3)\leq I(\mathbf{Z}_1;\mathbf{Z}_2)$ holds. As proved by Zhu et al. \cite{Zhu_2021}, since the node embeddings of two views  $\mathbf{H}_1$ and $\mathbf{H}_2$ are conditionally independent given the node embeddings of the original graph $\mathbf{H}$, they also satisfy the Markov relation with $\mathbf{H}_1\rightarrow\mathbf{H}\rightarrow\mathbf{H}_2$, and vice versa. Therefore, we can derive the following properties from their mutual information
\begin{align}
    I(\mathbf{H}_1; \mathbf{H}_2) &\leq I(\mathbf{H};\mathbf{H}_1), \\
    I(\mathbf{H}_1; \mathbf{H}_2) &\leq I(\mathbf{H};\mathbf{H}_2).
\end{align}
In fact, this inequality holds on each node. A sketch of the proof is that the embedding of each node $v_i$ is determined by all the nodes from its $l$-hop neighborhood if an $l$-layer GNN is used as the encoder, and this subgraph composed of its $l$-hop neighborhood also satisfies the Markov relation. Therefore, we can derive the more strict inequalities 
\begin{align}
\label{info_ineq1}
    I(\mathbf{H}_1[i,:]; \mathbf{H}_2[i,:]) &\leq I(\mathbf{H}[i,:];\mathbf{H}_1[i,:]), \\
\label{info_ineq2}
     I(\mathbf{H}_1[i,:]; \mathbf{H}_2[i,:]) &\leq I(\mathbf{H}[i,:];\mathbf{H}_2[i,:]).
\end{align}
Since $-L_{\text{con}}(G_1,G_2)$ is only a lower bound of the mutual information, directly applying the above constraints is hard, we only consider the constraints on the density ratio. Using the Markov relation for each node, we give the following theorem
\begin{theorem}
For two graph views $G_1$ and $G_2$ independently transformed from the graph $G$, the density ratio of their node embeddings $\mathbf{H}_1$ and $\mathbf{H}_2$ should satisfy $g(\mathbf{H}_2[i,:], \mathbf{H}_1[i,:])\leq g(\mathbf{H}_2[i,:], \mathbf{H}[i,:])$ and $g(\mathbf{H}_1[i,:], \mathbf{H}_2[i,:])\leq g(\mathbf{H}_1[i,:], \mathbf{H}[i,:])$, where $\mathbf{H}$ is the node embeddings of the original graph.
\end{theorem}
\begin{proof}
Following the Markov relation of each node, we get
\begin{align}
    p(\mathbf{H}_2[i,:]|\mathbf{H}_1[i,:])&=p(\mathbf{H}_2[i,:]|\mathbf{H}[i,:])p(\mathbf{H}[i,:]|\mathbf{H}_1[i,:])\\
    &\leq p(\mathbf{H}_2[i,:]|\mathbf{H}[i,:]),
\end{align}
and consequently
\begin{align}
    \frac{p(\mathbf{H}_2[i,:]|\mathbf{H}_1[i,:])}{p(\mathbf{H}_2[i,:])}\leq \frac{p(\mathbf{H}_2[i,:]|\mathbf{H}[i,:])}{p(\mathbf{H}_2[i,:])}.
\end{align}
Since $g(\mathbf{a},\mathbf{b}) \propto \frac{p(\mathbf{a}|\mathbf{b})}{p(\mathbf{a})}$, we get $g(\mathbf{H}_2[i,:], \mathbf{H}_1[i,:])\leq g(\mathbf{H}_2[i,:], \mathbf{H}[i,:])$. A similar proof applies to the other inequality.
\end{proof}
 Note that $g(\cdot,\cdot)$ is symmetric for the two inputs, we thus get two upper bounds for $g(\mathbf{H}_1[i,:], \mathbf{H}_2[i,:])$. According to the previous definition, $g(\mathbf{a},\mathbf{b})=e^{\theta(\mathbf{a},\mathbf{b})/\tau}$, we can simply replace $g(\cdot, \cdot)$ with $\theta(\cdot,\cdot)$ in the inequalities, then we combine these two upper bounds into one
\begin{align}
    2\cdot\theta(\mathbf{H}_1[i,:], \mathbf{H}_2[i,:]) \leq \theta(\mathbf{H}_2[i,:], \mathbf{H}[i,:])+\theta(\mathbf{H}_1[i,:], \mathbf{H}[i,:]).
\end{align}
This bound intuitively requires the similarity between $\mathbf{H}_1[i,:]$ and $\mathbf{H}_2[i,:]$ to be smaller than the similarity between $\mathbf{H}[i,:]$ and $\mathbf{H}_1[i,:]$ or $\mathbf{H}_2[i,:]$. 
Therefore, we define the following information regularization to penalize the higher probability of a less similar pair
\begin{align}
d_i = 2\cdot\theta(\mathbf{H}_1[i,:], \mathbf{H}_2[i,:]) &-(\theta(\mathbf{H}_2[i,:], \mathbf{H}[i,:])+\theta(\mathbf{H}_1[i,:], \mathbf{H}[i,:]))\\
    L_{I}(G_1, G_2, G) &= \frac{1}{n}\sum\limits_{i=1}^n{\max\{d_i, 0\}}.
\end{align}
Specifically, the information regularization could be defined over any three graphs that satisfy the Markov relation, but for our framework, to save the memory and time complexity, we avoid additional sampling and directly ground the information regularization on the existing graphs.

The final loss of \ariel\ can be written as 
\begin{align}
\label{eq:ARIEL}
    L(G_1,G_2, G_{\text{adv}}) = L_{\text{con}}(G_1,G_2)+\epsilon_1 L_{\text{con}}(G_1,G_{\text{adv}}) + \epsilon_2  L_{I}(G_1, G_2, G) ,
\end{align}
where $\epsilon_2 >0$ controls the strength of the information regularization.

The entire algorithm of \ariel\ is summarized in Algorithm \ref{alg:algorithm}.
\begin{algorithm}[tb]
\caption{Algorithm of \ariel}
\label{alg:algorithm}
\begin{algorithmic}[0] 
\STATE \textbf{Input data:} Graph $G=(\mathbf{A},\mathbf{X})$
\STATE \textbf{Input parameters:} $\alpha$, $\beta$, $\Delta_{\mathbf{A}}$, $\delta_{\mathbf{X}}$, $\epsilon_1$, $\epsilon_2$, $\gamma$ and $T$
\STATE Randomly initialize the graph encoder $f$
\FOR{iteration  $k=0,1,\cdots$}
\STATE Sample a subgraph $G_s$ from $G$
\STATE Generate two views $G_1$ and $G_2$ from $G_s$
\STATE Generate the adversarial view $G_{\text{adv}}$ according to Equations (\ref{eq:pgd_x}), (\ref{eq:pgd_a})
\STATE Update model $f$ to minimize $L(G_1,G_2, G_{\text{adv}})$ in Equation (\ref{eq:ARIEL})
\IF{$(k+1)\mod T =0$}
\STATE Update $\epsilon_1\leftarrow \gamma*\epsilon_1$
\ENDIF
\ENDFOR
\STATE \textbf{return:} Node embedding matrix  $\mathbf{H}=f(\mathbf{A},\mathbf{X})$ 
\end{algorithmic}
\end{algorithm}

%% file: experiment.tex
In this section, we conduct empirical evaluations, which are designed to answer the following three questions:
\begin{itemize}
    \item [RQ1.] How effective is the proposed \ariel\ in comparison with previous graph contrastive learning methods on the node classification task?
    \item [RQ2.] To what extent does \ariel\ gain robustness over the attacked graph?
    \item [RQ3.] How does each part of \ariel\ contribute to its performance?
\end{itemize}
We evaluate our method with the node classification task on both the real-world graphs and attacked graphs. The node embeddings are first learnt by the proposed \ariel\ algorithm, then the embeddings are fixed to perform the node classification with a logistic regression classifier trained over it. All our experiments are conducted on the NVIDIA Tesla V100S GPU with 32G memory.

\subsection{Experimental Setup}

\subsubsection{Datasets} We use six  datasets for the evaluation, including \textit{Cora}, \textit{CiteSeer}, \textit{Amazon-Computers}, \textit{Amazon-Photo}, \textit{Coauthor-CS} and \textit{Coauthor-Physics}. 
A summary of the dataset statistics is in Table \ref{tab:dataset}.
\begin{table}
\centering
\begin{tabular}{lcccc}
\toprule
Dataset & Nodes & Edges & Features & Classes \\
\midrule
Cora       & 2,708  & 5,429 &1,433 &7    \\
CiteSeer    &3,327  & 4,732 &3,703 &6     \\
Amazon-Computers & 13,752 & 245,861&767&10\\
Amazon-Photo & 7,650 &119,081&745&8\\
Coauthor-CS & 18,333&81,894&6,805&15\\
Coauthor-Physics&34,493&247,962&8,415&5\\
\bottomrule
\end{tabular}
\vspace{5pt}
\caption{Datasets statistics, number of nodes, edges, node feature dimensions, and classes are listed.}
\vspace*{-10pt}
\label{tab:dataset}
\end{table}

\subsubsection{Baselines} We consider six graph contrastive learning methods, including DeepWalk \cite{Perozzi:2014:DOL:2623330.2623732},  DGI \cite{velickovic2018deep},  GMI \cite{peng2020graph}, MVGRL \cite{hassani2020contrastive}, GRACE \cite{zhu2020deep} and GCA \cite{Zhu_2021}. Since DeepWalk only generates the embeddings for the graph topology, we concatenate the node features to the generated embeddings for the evaluation so that the final embeddings can incorporate both the topology and feature information. Besides, we also compare our method with two supervised methods Graph Convolutional Network (GCN) \cite{kipf2017semisupervised} and Graph Attention Network (GAT) \cite{velickovic2018graph}. 

\begin{table*}[h]
\centering
\begin{tabular}{ccccccc}
\toprule
\textbf{Method}  & \textbf{Cora} & \textbf{CiteSeer} & \textbf{Amazon-Computers} & \textbf{Amazon-Photo} & \textbf{Coauthor-CS} & \textbf{Coauthor-Physics} \\
\midrule
GCN &$84.14\pm0.68$&$69.02\pm0.94$&$88.03\pm1.41$&$92.65\pm0.71$&$92.77\pm0.19$&$95.76\pm0.11$\\
GAT&$83.18\pm1.17$&$69.48\pm1.04$&$85.52\pm2.05$&$91.35\pm1.70$&$90.47\pm0.35$&$94.82\pm 0.21$\\
\midrule
DeepWalk+features&$79.82\pm0.85$&$67.14\pm0.81$&$86.23\pm0.37$&$90.45\pm0.45$&$85.02\pm0.44$&$94.57\pm0.20$\\
DGI     &$84.24\pm0.75$   & $69.12\pm1.29$ & $88.78\pm0.43$   & $92.57\pm0.23$&$92.26\pm0.12$&$95.38\pm0.07$\\
GMI &$82.43\pm0.90$&$70.14\pm1.00$&$83.57\pm0.40$&$88.04\pm0.59$&OOM&OOM\\
MVGRL   &\textbf{84.39 $\pm$ 0.77}   &$69.85\pm1.54$& $89.02\pm0.21$&$92.92\pm0.25$ & $92.22\pm0.22$&$95.49\pm0.17$\\

GRACE   &$83.40\pm1.08$ & $69.47\pm1.36$  &  $87.77\pm0.34$ & $92.62\pm0.25$ &$93.06\pm 0.08$ &$95.64\pm0.08$\\
GCA-DE &$82.57\pm0.87$& $72.11\pm0.98$ & $88.10\pm0.33$ &$92.87\pm0.27$ &$93.08\pm0.18$&$95.62\pm0.13$\\
GCA-PR &$82.54\pm0.87$& $72.16\pm0.73$ & $88.18\pm0.39$ &$92.85\pm0.34$ &$93.09\pm0.15$&$95.58\pm0.12$\\
GCA-EV &$81.80\pm0.92$& $67.07\pm0.79$ & $87.95\pm0.43$ &$92.63\pm0.33$ &$93.06\pm0.14$&$95.64\pm0.08$\\
\midrule
\textbf{ARIEL}  & 84.28 $\pm$ 0.96 & \textbf{72.74 $\pm$ 1.10} &\textbf{91.13 $\pm$ 0.34} &\textbf{94.01 $\pm$ 0.23}&\textbf{93.83 $\pm$ 0.14}&\textbf{95.98 $\pm$ 0.05}    \\
\bottomrule
\end{tabular}
\vspace{5pt}
\caption{Node classification accuracy in percentage on six real-world datasets. We bold the results with the best mean accuracy. The methods above the line are the supervised ones, and the ones below the line are unsupervised. OOM stands for Out-of-Memory on our 32G GPUs.}
\vspace*{-10pt}
\label{tab:results}
\end{table*}

\subsubsection{Evaluation protocol} For each dataset, we randomly select $10\%$ nodes for training, $10\%$ nodes for validation and the remaining for testing. For contrastive learning methods, a logistic regression classifier is trained to do the node classification over the node embeddings. The accuracy is used as the evaluation metric. 

Besides testing on the original, clean graphs, we also evaluate our method on the attacked graphs. We use Metattack \cite{zugner2018adversarial} to perform the poisoning attack. 
For \ariel\, we use the hyperparameters of the best models we obtain on the clean graphs for evaluation. For GCA, we report the performance for its three variants, GCA-DE, GCA-PR, and GCA-EV, which correspond to the adoption of degree, eigenvector and PageRank \cite{ilprints422, kang2018aurora} centrality measures, in our main results and use the best variant on each dataset for the evaluation on the attacked graphs.
\subsubsection{Hyperparameters}
For \ariel\, we use the same parameters and design choices for the network
architecture, optimizer, and training scheme as in GRACE and GCA on each dataset. However, we find GCA not behave well on Cora with a significant performance drop, so we re-search the parameters for GCA on Cora separately and use a different temperature for it. Other contrastive learning-specific parameters are kept the same over GRACE, GCA, and \ariel.  

All GNN-based baselines use a two-layer GCN as the encoder. For each method, we compare its default hyperparameters and the ones used by \ariel, and use the hyperparameters leading to better performance. Other algorithm-specific hyperparameters all respect the default setting in its official implementation.

Other hyperparameters of \ariel\ are summarized in the Appendix.
\begin{table*}
\centering
\begin{tabular}{ccccccc}
\toprule
\textbf{Method}  & \textbf{Cora} & \textbf{CiteSeer} & \textbf{Amazon-Computers} & \textbf{Amazon-Photos} & \textbf{Coauthor-CS} & \textbf{Coauthor-Physics} \\
\midrule
GCN &$80.03\pm0.91$&$62.98\pm1.20$& $84.10\pm1.05$ & $91.72\pm0.94$ & $80.32\pm0.59$ & $87.47\pm0.38$\\
GAT &$79.49\pm1.29$&$63.30\pm1.11$&$81.60\pm1.59$&$90.66\pm 1.62$ & $77.75\pm0.80$& $86.65\pm0.41$\\
\midrule
DeepWalk+features &$74.12\pm1.02$&$63.20\pm0.80$&$79.08\pm0.67$&$88.06\pm0.41$& $49.30\pm 1.23$ & $79.26\pm1.38$\\
DGI &$80.84\pm0.82$&$64.25\pm0.96$ & $83.36\pm 0.55$ & $91.27\pm0.29$&$78.73\pm0.50$ & $85.88\pm0.37$   \\
GMI &79.17 $\pm$ 0.76&$65.37\pm1.03$& $77.42\pm0.59$& $89.44\pm0.47$& $80.92\pm0.64$ & $87.72\pm0.45$\\
MVGRL & \textbf{80.90 $\pm$ 0.75}&$64.81\pm1.53$ & $83.76\pm0.69$ & $91.76\pm0.44$ & $79.49 \pm 0.75$ & $86.98\pm0.61$\\
GRACE &$78.55\pm0.81$ & $63.17\pm 1.81$ & $84.74\pm1.13$ & $91.26\pm0.37$ &$80.61\pm0.63$ & $85.71\pm0.38$ \\
GCA &$76.79\pm0.97$&$64.89\pm1.33$& $85.05\pm0.51$ &$91.71\pm0.34$&$82.72\pm0.58$&$89.00\pm0.31$\\
\midrule
\textbf{ARIEL}  &  80.33 $\pm$ 1.25 & \textbf{69.13 $\pm$ 0.94} & \textbf{88.61 $\pm$ 0.46} & \textbf{92.99 $\pm$ 0.21} & \textbf{84.43 $\pm$ 0.59} & \textbf{89.09 $\pm$ 0.31} \\
\bottomrule
\end{tabular}
\vspace{5pt}
\caption{Node classification accuracy in percentage on the graphs under Metattack, where subgraphs of Amazon-Computers, Amazon-Photo, Coauthor-CS, and Coauthor-Physics are used for attack and their results are not directly comparable to those in Table \ref{tab:results}. We bold the results with the best mean accuracy. GCA is evaluated on its best variant on each clean graph.}
\vspace*{-10pt}
\label{tab:att_results}
\end{table*}
\subsection{Main Results}
The comparison results of node classification on all six datasets are summarized in Table \ref{tab:results}. 

Our method \ariel\ outperforms baselines over all datasets except on Cora, with only $0.11\%$ lower in accuracy than MVGRL. It can be seen that the state-of-the-art method GCA does not bear significant improvements over previous methods.
In contrast, \ariel\ can achieve consistent improvements over GRACE and GCA on all datasets, especially on Amazon-Computers with almost $3\%$ gain. 

Besides, we find MVGRL a solid baseline whose performance is close to or even better than GCA on these datasets. It achieves the highest score on Cora and the second-highest on Amazon-Computers and Amazon-Photo. However, it does not behave well on CiteSeer, where GCA can effectively increase the score of GRACE. To sum up, previous modifications over the grounded frameworks are mostly based on specific knowledge, for example, MVGRL introduces the diffusion matrix to DGI and GCA defines the importance on the edges and features, and they cannot consistently take effect on all datasets. However, \ariel\ uses the adversarial attack to automatically construct the high-quality contrastive samples and achieves more stable performance improvements.

Finally, in comparison with the supervised methods, \ariel\ also achieves a clear advantage over all of them. Although it would be premature to conclude that \ariel\ is more powerful than these supervised methods since they are usually tested under the specific training-testing split, these results do demonstrate that \ariel\ can indeed generate highly expressive node embeddings for the node classification task, which can achieve comparable performance to the supervised methods.

\subsection{Results under Attack}
 The results on attacked graphs are summarized in Table \ref{tab:att_results}. Specifically, we evaluate all these methods on the attacked subgraph of Amazon-Computers, Amazon-Photo, Coauthor-CS, and Coauthor-Physics, so their results are not directly comparable to the results in Table \ref{tab:results}.

It can be seen that although some baselines are robust on specific datasets, for example, MVGRL on Cora, GMI on CiteSeer, and GCA on Coauthor-CS and Coauthor-Physics, they fail to achieve consistent robustness over all datasets. Although GCA indeed makes GRACE more robust, it is still vulnerable on Cora, CiteSeer, and Amazon-Computers, with more than $3\%$ lower than \ariel\ in the final accuracy. 

Basically, MVGRL and GCA can improve the robustness of their respective grounded frameworks over different datasets, but we find this kind of improvement relatively minor. Instead, \ariel\ has more significant improvements and greatly increases robustness.

\subsection{Ablation Study}

For this section, we first set $\epsilon_2$ as $0$ and investigate the role of adversarial contrastive loss. The adversarial contrastive loss coefficient $\epsilon_1$ controls the portion of the adversarial contrastive loss in the final loss. When $\epsilon_1=0$, the final loss reduces to the regular contrastive loss in Equation (\ref{eq:contra}). To explore the effect of the adversarial contrastive loss, we fix other parameters in our best models on Cora and CiteSeer and gradually increase $\epsilon_1$ from $0$ to $2$. The changes in the final performance are shown in Figure \ref{fig:eps1}.

The dashed line represents the performance of GRACE with subgraph sampling, i.e., $\epsilon_1=0$. Although there exist some variations, \ariel\ is always above the baseline under a positive $\epsilon_1$ with around $2\%$ improvement. The subgraph sampling trick may sometimes help the model, for example, it improves GRACE without subgraph sampling by $1\%$ on CiteSeer, but it could be detrimental as well, such as on Cora. This is understandable since subgraph sampling can simultaneously enrich the data augmentation and lessen the number of negative samples, both critical to contrastive learning. While for the adversarial contrastive loss, has a stable and significant improvement on GRACE with subgraph sampling, which demonstrates that the performance improvement of \ariel\ mainly stems from the adversarial loss rather than the subgraph sampling.
\begin{figure}[h]
    \begin{minipage}{0.48\linewidth}
    \includegraphics[width=\linewidth]{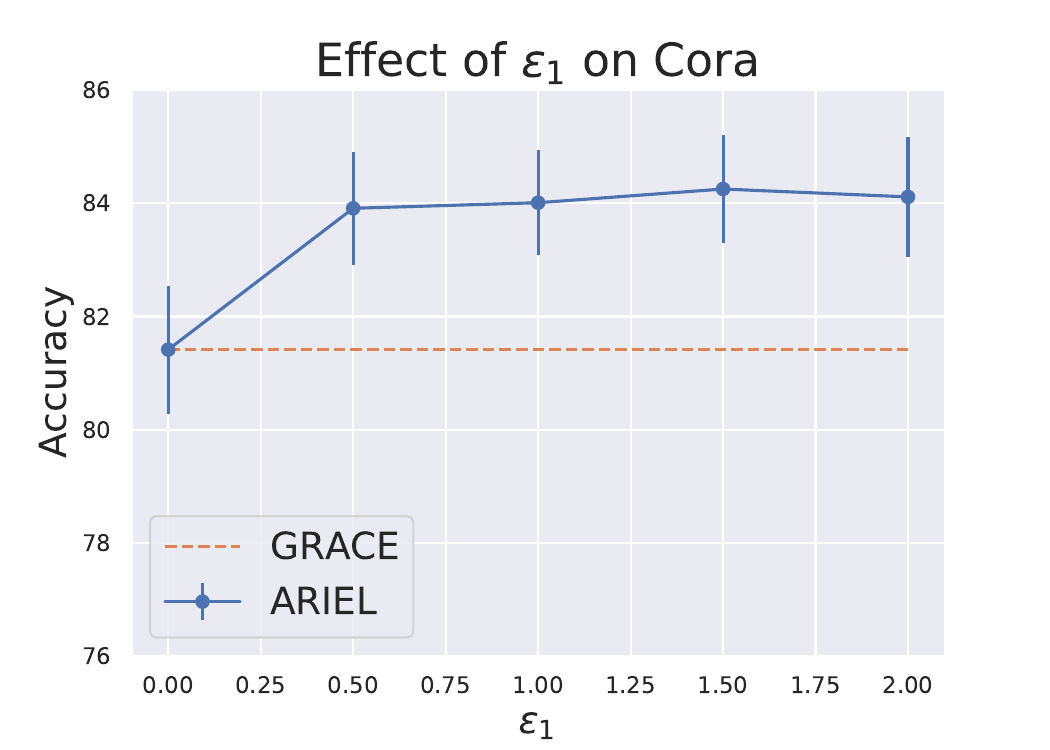}
    \end{minipage}
    \begin{minipage}{.48\linewidth}
    \includegraphics[width=\linewidth]{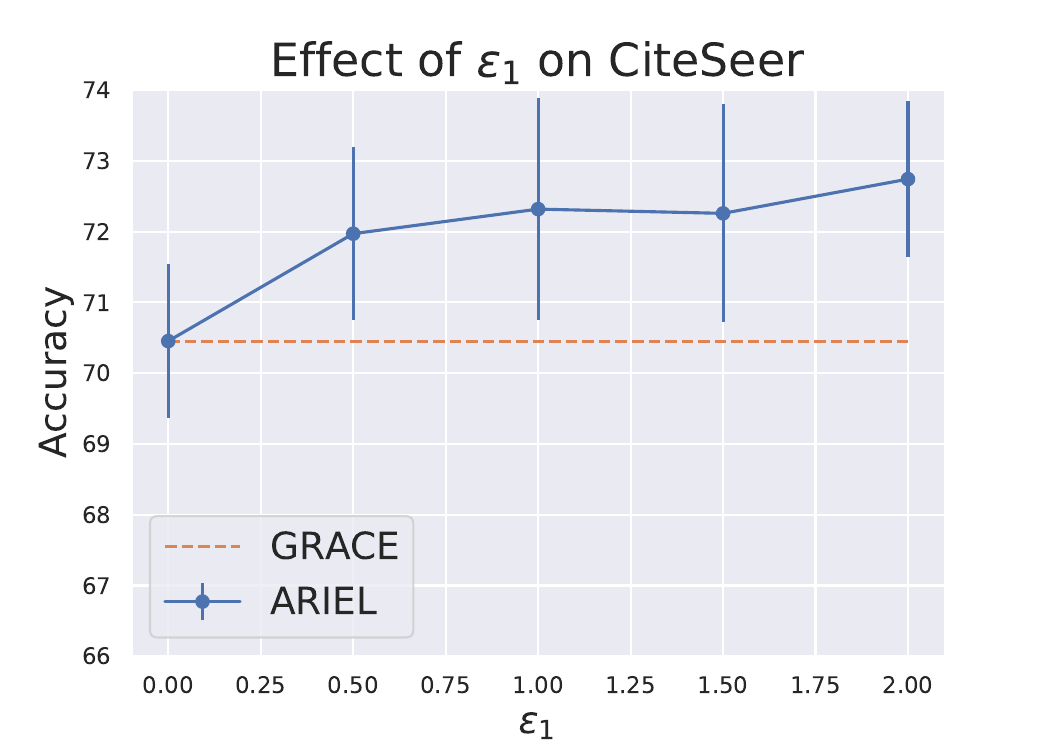}
    \end{minipage}
    \caption{Effect of adversarial contrastive loss coefficient $\epsilon_1$ on Cora and CiteSeer. The dashed line represents the performance of GRACE with subgraph sampling.}
    \vspace{-10pt}
    \label{fig:eps1}
\end{figure}

Next, we fix all other parameters and check the behavior of $\epsilon_2$. Information regularization is mainly designed to stabilize the training of \ariel. We find \ariel\ would experience the collapsing at the early training stage and the information regularization could mitigate this issue. We choose the best run on Amazon-Photo, where the collapsing frequently occurs, and similar to before, we gradually increase $\epsilon_2$ from $0$ to $2$, the results are shown in Figure \ref{fig:eps2} (left).
\begin{figure}[h]
    \begin{minipage}{0.48\linewidth}
    \includegraphics[width=\linewidth]{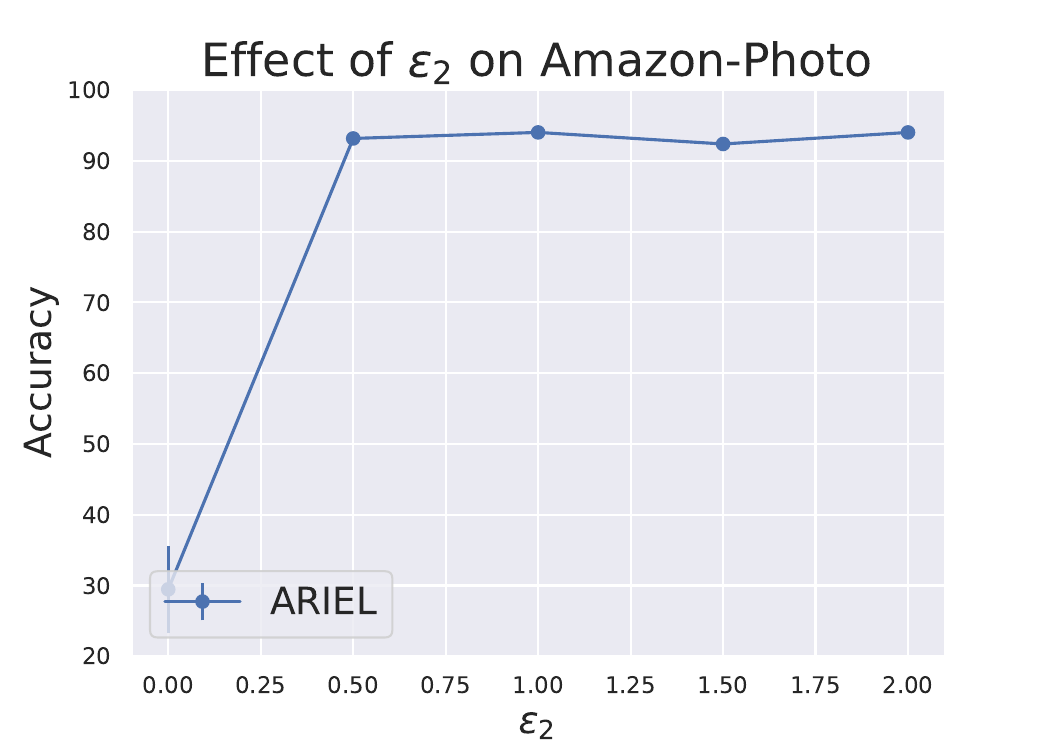}
    \end{minipage}
    \begin{minipage}{.48\linewidth}
    \includegraphics[width=\linewidth]{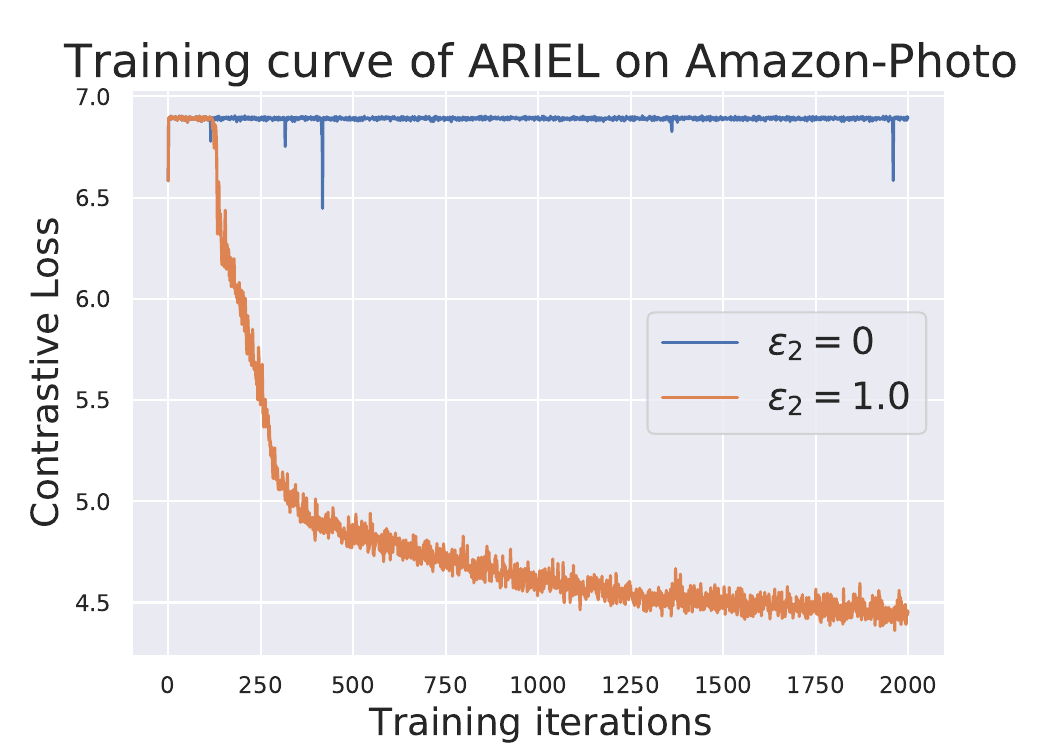}
    \end{minipage}
    \caption{Effect of information regularization on Amazon-Photo. The left figure shows the model performance under different $\epsilon_2$ and the right figure plots the training curve of \ariel\ under $\epsilon_2=0$ and $\epsilon_2=1.0$.}
    \vspace{-5pt}
    \label{fig:eps2}
\end{figure}
As can be seen, without using the information regularization, \ariel\ could collapse without learning anything, while setting $\epsilon_2$ greater than $0$ can effectively avoid this situation. To further illustrate this, we draw the training curve of the regular contrastive loss in Figure \ref{fig:eps2} (right), for the best \ariel\ model on Amazon-Photo and the same model by simply removing the information regularization. Without information regularization, the model could get stuck in a bad parameter and fail to converge, while information regularization can resolve this issue.

%% file: related.tex
In this section, we review the related work in the following three categories: graph contrastive learning, adversarial attack on graphs, and adversarial contrastive learning.

\subsection{Graph Contrastive Learning}
Contrastive learning is known for its simplicity and strong expressivity. Traditional methods ground the contrastive samples on the node proximity in the graph, such as DeepWalk \cite{Perozzi:2014:DOL:2623330.2623732} and node2vec \cite{grover2016node2vec}, which use random walks to generate the node sequences and approximate the co-occurrence probability of node pairs. However, these methods can only learn the embeddings for the graph structures regardless of the node features. 

GNNs \cite{kipf2017semisupervised, velickovic2018graph} can easily capture the local proximity and node features. To further improve the performance, the Information Maximization (InfoMax) principle \cite{linsker1988self} has been introduced. DGI \cite{velickovic2018deep} is adapted from Deep InfoMax \cite{hjelm2019learning} to maximize the mutual information between the local and global features. It generates the negative samples with a corrupted graph and contrasts the node embeddings with the original graph embedding and the corrupted one. Based on a similar idea, GMI \cite{peng2020graph} generalizes the concept of mutual information to the graph domain by separately defining the mutual information on the features and edges. The other follow-up work of DGI, MVGRL \cite{hassani2020contrastive}, maximizes the mutual information between the first-order neighbors and graph diffusion. 
HDMI \cite{jing2021hdmi} introduces high-order mutual information to consider both intrinsic and extrinsic training signals.
However, mutual information-based methods generate the corrupted graphs by simply randomly shuffling the node features.
Recent methods exploit the graph topology and features to generate better-augmented graphs. GCC \cite{Qiu_2020} adopts a random-walk-based strategy to generate different views of the context graph for a node, but it ignores the augmentation of the feature level. GCA \cite{zhu2020deep}, instead, considers the data augmentation from both the topology and feature level, and introduces the adaptive augmentation by considering the importance of each edge and feature dimension. Unlike the above methods which construct the data augmentation samples based on domain knowledge, \ariel\ uses adversarial attacks to construct the view that maximizes the contrastive loss, which is more informative with broader applicability.

\subsection{Adversarial Attack on Graphs}
Deep learning methods on graphs are known vulnerable to adversarial attacks. As shown by Bojchevski et al. \cite{bojchevski2019adversarial}, both random-walk-based methods and GNN-based methods could be attacked by flipping a small portion of edges. Xu et al. \cite{xu2019topology} propose a PGD attack and min-max attack on the graph structure from the optimization perspective. NETTACK \cite{Z_gner_2018} is the first to attack GNNs using both structure attack and feature attack, causing a significant performance drop of GNNs on the benchmarks. After that, Metattack \cite{zugner2018adversarial}  formulates the poisoning attack of GNNs as a meta-learning problem and achieves remarkable performance by only perturbing the graph structure. Node Injection Poisoning Attacks \cite{NIPA} uses a hierarchical reinforcement learning approach to sequentially manipulate the labels and links of the injected nodes. Recently, InfMax \cite{ma2021adversarial} formulates the adversarial attack on GNNs as an influence maximization problem.

\subsection{Adversarial Contrastive Learning}
The concept of adversarial contrastive learning is first proposed on visual domains \cite{kim2020adversarial, jiang2020robust, ho2020contrastive}. All these works propose a similar idea to use the adversarial sample as a form of data augmentation in contrastive learning and it can bring a better downstream task performance and higher robustness. Recently, AD-GCL \cite{suresh2021adversarial} formulates adversarial graph contrastive learning in a min-max form and uses a parameterized network for edge dropping. However, AD-GCL is designed for the graph classification task instead and it does not explore the robustness of graph contrastive learning. Moreover, both visual adversarial contrastive learning methods and AD-GCL deals with independent instances, e.g., independent images or graphs, 
 while the instances studied by \ariel\ have interdependence, which makes scalability a challenging issue.

%% file: conclusion.tex
In this paper, we propose a simple yet effective framework for graph contrastive learning by introducing an adversarial view and we stabilize it through the information regularization. It consistently outperforms the state-of-the-art graph contrastive learning methods in the node classification task and exhibits a higher degree of robustness to the adversarial attack. Our framework is not limited to the graph contrastive learning framework we build on in this paper, and it can be naturally extended to other graph contrastive learning methods as well.  In the future, we plan to further investigate (1) the adversarial attack on graph contrastive learning and (2) the integration of graph contrastive learning and supervised methods.

%% file: appendix.tex
\section{Additional Experimental Setup Details}
In this section, we describe some experimental setup details not covered in our main text.

The six datasets used in this paper, including \textit{Cora}, \textit{CiteSeer}, \textit{Amazon-Computers}, \textit{Amazon-Photo}, \textit{Coauthor-CS} and \textit{Coauthor-Physics}, are from Pytorch Geometric \footnote{All the datasets are from Pytorch Geometric 1.6.3: \url{https://pytorch-geometric.readthedocs.io/en/latest/modules/datasets.html}}.
Cora and CiteSeer \cite{yang2016revisiting} are citation networks, where nodes represent documents and edges correspond to citations. Amazon-Computers and Amazon-Photo \cite{shchur2019pitfalls} are extracted from the Amazon co-purchase graph. In these graphs, nodes are the goods and they are connected by an edge if they are frequently bought together. Coauthor-CS and Coauthor-Physics are \cite{shchur2019pitfalls} the co-authorship graphs, where each node is an author and the edge indicates the co-authorship on a paper. 

For the evaluation, different from the logistic regression implementation used by DGI \cite{velickovic2018deep}, we use the implementation from scikit-learn \footnote{\url{https://scikit-learn.org/stable/}} directly since it has a better performance.

We search each method over $6$ different random seeds, including  $5$ random seeds from our own and the best random seed of GCA on each dataset. For each seed, we evaluate the method on $20$ random training-validation-testing dataset splits and report the mean and the standard deviation of the accuracy on the best seed. Specifically, for the supervised learning methods, we abandon the existing splits, for example on Cora and CiteSeer, but instead do a random split before the training and report the results over $20$ splits.

On the attacked graphs, since Metattack \cite{zugner2018adversarial} is targeted at graph structure only and computationally inefficient on the large graphs, we first randomly sample a subgraph of $5000$ nodes if the number of nodes in the original graph is greater than $5000$, then we randomly mask out $20\%$ features, and finally use Metattack to perturb $20\%$ edges to generate the final attacked graph.

\section{Hyperparameters of \ariel}
In this section, we give the details about the hyperparameter searching for \ariel\ in our experiment, they are summarized as follows:
\begin{itemize}
    \item Adversarial contrastive loss coefficient $\epsilon_1$ and information regularization strength $\epsilon_2$. We search them over $\{0.5, 1, 1.5, 2\}$ and use the one with the best performance on the validation set 
    of each dataset. Specifically, we first fix $\epsilon_2$ as $0$ and decide the optimal value for other parameters, then we search $\epsilon_2$ on top of the model with all other hyperparameters fixed.
    \item Number of attack steps and perturbation constraint. These parameters are fixed on all datasets, with the number of attack steps 5, edge perturbation constraint $\Delta_{\mathbf{A}}=0.1\sum_{i,j}\mathbf{A}[i,j]$ and feature perturbation constraint $\delta_{\mathbf{X}}=0.5$.
    \item Curriculum learning weight $\gamma$ and change period $T$. In our experiments, we simply fix $\gamma=1.1$ and $T=20$.
    \item Graph perturbation rate $\alpha$ and feature perturbation rate $\beta$. We search both of them over $\{0.001, 0.01, 0.1\}$ and take the best one on the validation set of each dataset.
    \item Subgraph size. We keep the subgraph size $500$ for \ariel\ on all datasets. Although we find increasing the subgraph size benefits \ariel\ on larger graphs, using this fixed size still ensures a solid performance of \ariel.
\end{itemize}

\section{Training Analysis}
Here we compare the training of \ariel\  to other methods on our NVIDIA Tesla V100S GPU with 32G memory.

Adversarial attacks on graphs tend to be highly computationally expensive since the attack requires the gradient calculation over the entire adjacency matrix, which is of size $O(n^2)$. For \ariel, we resolve this bottleneck with subgraph sampling and empirically show that the adversarial training on the subgraph still yields significant improvements, without increasing the number of training iterations. In our experiments, we find GMI the most memory inefficient, which cannot be trained on Coauthor-CS and Coauthor-Physics. For DGI, MVGRL, GRACE, and GCA, the training of them also amounts to $30$G GPU memory on Coauthor-Physics while the training of \ariel\ requires no more than $8$G GPU memory. In terms of the training time, DGI and MVGRL are the fastest to converge, but it takes MVGRL a long time to compute the diffusion matrix on large graphs. \ariel\ is slower than GRACE and GCA on Cora and CiteSeer, but it is faster on large graphs like  Coauthor-CS and Coauthor-Physics, with the training time for each iteration invariant to the graph size due to the subgraph sampling. On the largest graph Coauthor-Physics, each iteration takes GRACE 0.875 second and GCA 1.264 seconds, while it only takes \ariel\ 0.082 second. This demonstrates that \ariel\ has even better scalability than GRACE and GCA.
\clearpage